\documentclass[letterpaper, 10 pt, conference]{ieeeconf}  

\IEEEoverridecommandlockouts                              

\usepackage{color}
\usepackage{graphicx,amsmath,amssymb,bm,xcolor}
\usepackage[font=small,skip=0pt]{caption}
\usepackage{subcaption} 
\usepackage[utf8]{inputenc}
\usepackage{breqn} 
\usepackage[noadjust]{cite}
\usepackage[ruled,linesnumbered,norelsize]{algorithm2e}
\newtheorem{theorem}{Theorem}
\usepackage[colorlinks,bookmarksopen,bookmarksnumbered,citecolor=red,urlcolor=red, bookmarks=true]{hyperref}
\usepackage{cleveref}

\Crefname{equation}{Eq.}{Eqs.}
\Crefname{figure}{Fig.}{Figs.}

\usepackage{dsfont}
\ifdefined\N                                                                
\renewcommand{\N}{\mathds{N}}                                                
\else
\newcommand{\N}{\mathds{N}}
\fi
\ifdefined\C
\renewcommand{\C}{\mathds{C}}                                             
\else
\newcommand{\C}{\mathds{C}}
\fi

\def\argmax{\mathop{\sf arg\,max}}                                          
\def\argmin{\mathop{\sf arg\,min}}                                          
\newcommand{\sign}{\operatorname{sign}}                                     
\newcommand{\abs}[1]{\left|#1\right|}

\newcommand{\id}{\boldsymbol{I}}                                                
\newcommand{\diag}{\operatorname{diag}}                                     

\newcommand{\matb}[1]{ 														
	\begin{bmatrix}
		#1
	\end{bmatrix}
}
\newcommand{\trsp}{{\scriptscriptstyle\top}}									
\newcommand{\inv}{{\raisebox{.2ex}{$\scriptscriptstyle-1$}}}

\newcommand{\nd}[1]{\bm{#1}}
\newcommand{\st}{\mbox{s.t.}}
\newcommand{\norm}[1]{\lVert #1 \rVert}
\newcommand{\bignorm}[1]{\left\lVert #1 \right\rVert}

\title{\LARGE \bf
Projection-based first-order constrained optimization solver for robotics
}

\author{Hakan Girgin, Tobias L\"ow, Teng Xue and Sylvain Calinon
\thanks{The authors are with the Idiap Research Institute, Martigny, Switzerland
and with EPFL, Lausanne, Switzerland.}
\thanks{This work was supported by the State Secretariat for Education,
Research and Innovation in Switzerland for participation in the European
Commission Horizon Europe Program through the INTELLIMAN
project (https://intelliman-project.eu/, HORIZON-CL4-Digital-Emerging
Grant 101070136) and the SESTOSENSO project (http://sestosenso.eu/,
HORIZON-CL4-Digital-Emerging Grant 101070310).}
}

\begin{document}

\maketitle
\thispagestyle{empty}
\pagestyle{empty}

\begin{abstract}
Robot programming tools ranging from inverse kinematics (IK) to model predictive control (MPC) are most often described as constrained optimization problems. Even though there are currently many commercially-available second-order solvers, robotics literature recently focused on efficient implementations and improvements over these solvers for real-time robotic applications. However, most often, these implementations stay problem-specific and are not easy to access or implement, or do not exploit the geometric aspect of the robotics problems. In this work, we propose to solve these problems using a fast, easy-to-implement first-order method that fully exploits the geometric constraints via Euclidean projections, called Augmented Lagrangian Spectral Projected Gradient Descent (ALSPG). We show that 1. using projections instead of full constraints and gradients improves the performance of the solver and 2. ALSPG stays competitive to the standard second-order methods such as iLQR in the unconstrained case. We showcase these results with IK and motion planning problems on simulated examples and with an MPC problem on a 7-axis manipulator experiment.
\end{abstract}

\section{Introduction}
Many tasks in robotics can be framed as constrained optimization problems. The inverse kinematics (IK) problem finds a configuration of the robot that corresponds to a desired pose in the task space while satisfying constraints such as joint limits or center-of-mass stability. Motion planning and optimal control determine a trajectory of configurations and/or control commands achieving the task subject to the dynamics and the constraints of the task and the environment over a certain time horizon. Model predictive control (MPC) recasts the optimal control problems with shorter horizons to solve simpler constrained optimization problems in real-time. In this work, we present a projection-based first-order optimization method that can be implemented and used for all these aforementioned problems.

There are many commercially available second-order solvers to address general constrained optimization problems such as SNOPT \cite{snopt}, SLSQP \cite{slsqp}, LANCELOT \cite{lancelot} and IPOPT \cite{ipopt}. In robotics, the literature on optimization mainly focuses on developing solvers that effectively solve each robotic problem separately. For example, in the motion planning literature, one can find many constrained variants of differential dynamic programming (DDP) \cite{constrainedDDP} or iterative  linear quadratic regulator (iLQR) \cite{tross}, TrajOpt \cite{trajopt}, CHOMP \cite{chomp} and ALTRO \cite{altro}. Furthermore, some of these solvers are not open-source and difficult to implement, which hinders benchmarking and potential improvements. As powerful as these solvers are, their applications for finding real-time feedback mechanisms such as closed-loop inverse kinematics and MPC requires tuning and adaptations of the solver. In this chapter, we address these challenges by proposing a very simple, yet powerful solver that can be easily implemented without having large memory requirements.

\begin{figure}
	\centering
	\subcaptionbox{\label{fig:ik_point}}
	[0.25\linewidth]{\includegraphics[width=0.25\columnwidth]{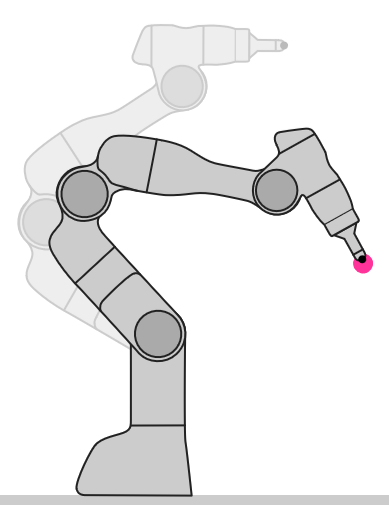}}
	\subcaptionbox{\label{fig:ik_line}}
[0.22\linewidth]{\includegraphics[width=0.22\columnwidth]{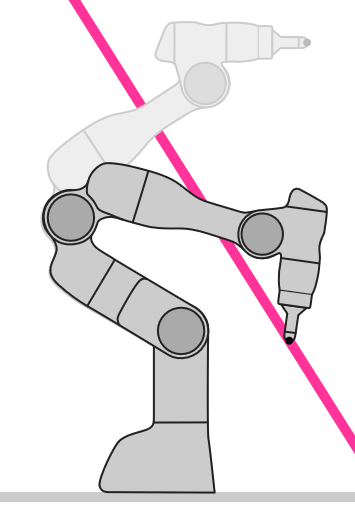}}
	\subcaptionbox{\label{fig:ik_circle}}
[0.252\linewidth]{\includegraphics[width=0.252\columnwidth]{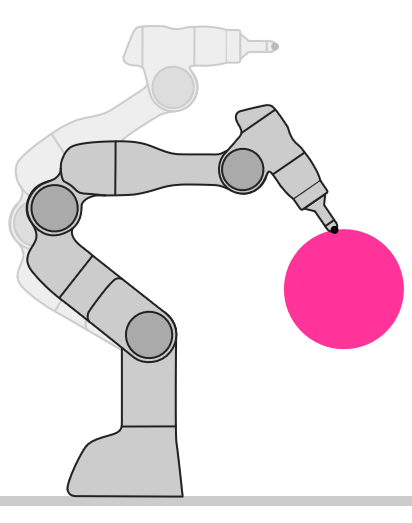}}
	\subcaptionbox{\label{fig:ik_square}}
[0.24\linewidth]{\includegraphics[width=0.24\columnwidth]{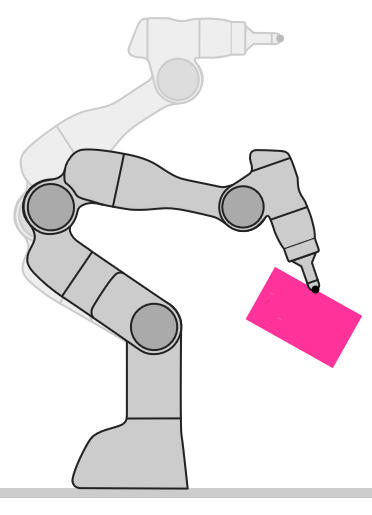}}
\caption{Projection view of inverse kinematics problem. 
(a) Reaching a point (standard IK problem): $\mathcal{C}_{\nd p}=\{\nd p \: | \: \nd p=\nd p_d\}$.
(b) Reaching under/above/on a plane (in the halfspace): $\mathcal{C}_{\nd p}=\{\nd p \: | \: \nd a^\trsp \nd p {\leq} \nd b\}$.
(c) Reaching inside/outside/on a circle: $\mathcal{C}_{\nd p}=\{\nd p \: | \: \norm{\nd p - \nd p_d}_2^2 {\leq} r^2\}$.
(d) Reaching inside/outside/on a rectangle: $\mathcal{C}_{\nd p}=\{\nd p \: | \: \norm{\nd A(\nd p - \nd p_d)}_{\infty} {\leq} L/2\}$.
These problems can be tested online with closed-loop controllers\protect\footnotemark, created as extensions of the \emph{Robotics Codes from Scratch} toolbox\protect\footnotemark. 
	}
\label{fig:ik_examples}
\end{figure}

The constraints in many of these problems are described as geometric set primitives or their combinations (see \Cref{table:projections}). Examples include joint angle or velocity limits or center-of-mass stability as bounded domain sets, avoiding/reaching geometric shapes such as spheres and convex polytopes as hyperplane and quadric sets, friction cone constraints as second-order cone sets. These constraints have in common that they can be formulated as projections rather than constraints. We argue that exploiting the projection capability of these sets instead of treating them as generic constraints in the solvers can significantly improve the performance.

Projected gradient descent is the simplest algorithm that takes into account these projections. Its idea is to project the gradient to have a next iterate inside the constraint set. In the optimization literature, a first-order projection-based solver called spectral projected gradient descent (SPG) has emerged as an alternative \cite{birgin2009spectral}. SPG has been studied and applied to many fields because of its great practical performance even compared to second-order constrained optimization solvers \cite{birgin2014spectral}. Its extension to additional arbitrary constraints has been proposed as within augmented Lagrangian methods \cite{andreani2008augmented, birgin2014practical, jia2022augmented}. However, as the application of this idea to popular second-order methods in robotics is not trivial \cite{schmidt2009optimizing}, the usefulness of projections in the field has been overlooked.

In this letter, we integrate the recent work of \cite{jia2022augmented} into robotics optimization problems ranging from IK to MPC by providing the most common Euclidean projections with an additional rectangular projection. We propose an extension with multiple projections and additional nonlinear constraints. In particular, we provide an efficient direct-shooting optimal control formulation of this solver to address motion planning and MPC problems.

\section{Related work}
\label{sec:related}
Euclidean projections and the analytical expressions to many projections can be found in \cite{projections}. It also gives a general theory on how to project onto a level set of an arbitrary function using KKT conditions. Extensive studies and theoretical background on the projections and their properties can be found in \cite{bauschke2011convex}. In \cite{usmanova2021fast}, Usmanova \emph{et al.} propose an efficient algorithm for the projection onto arbitrary convex constraint sets and show that exploiting projections in the optimization significantly increases the performance. In \cite{bauschke2015projection}, Bauschke and Koch discuss and benchmark algorithms for finding the projection onto the intersection of convex sets. One of the main algorithms for this is Dykstra's alternating projection algorithm \cite{dykstraproject}.

The simplest algorithm that exploits projections is the projected gradient descent. Spectral projected gradient descent improves over this by exploiting the curvature information via its spectral stepsizes. A detailed review on spectral projected gradient methods is given in \cite{birgin2014practical}.

In \cite{torrisi2018projected}, Torrisi \emph{et al.} propose to use a projected gradient descent algorithm to solve the subproblems of sequential quadratic programming (SQP). They show that their method can solve MPC of an inverted pendulum faster than SNOPT. Our work is closest to theirs with the differences that we use SPG instead of a vanilla projected gradient descent to solve the subproblems of augmented Lagrangian instead of SQP. Instead, we propose a direct way of handling multiple projections and inequality constraints, which is not trivial in \cite{torrisi2018projected}. 

In \cite{giftthaler2017projection}, Giftthaler and Buchli propose a projection of the update direction of the control input onto the nullspace of the linearized constraints in iLQR. This approach can only handle simple equality constraints (for example, velocity-level constraints of second-order systems) and cannot treat position-level constraints for such systems, which is a very common and practical class of constraints in real world applications.

\section{Background}
\label{sec:motivation}
\begin{table*}[tph!]
	\caption{Projections onto bounded domain, affine hyperplane, quadric and second-order cone }
	\centering
	\setlength\tabcolsep{2pt}
	\begin{tabular}{| c | c | c | c | c | c |}
		\cline{2-5}
		\multicolumn{1}{c|}{}       
		& Bounds
		& Affine hyperplane
		& Quadric
		& Second-order cone
		\\ \hline
		$\mathcal{C}$ 
		& $l{\leq} x{\leq} u$
		& $l{\leq}\nd a^\trsp \nd x {\leq} u$
		& $l{\leq}\frac{1}{2}\nd x^\trsp\nd x {\leq} u$
		& $\norm{\nd x} {\leq} t$
		\\ \hline
		$\Pi_{\mathcal{C}}$ &
		$ \begin{cases}
			x & \!\!\text{if } l{\leq} x {\leq} u\\
			u & \!\!\text{if }  x {>} u\\
			l & \!\!\text{if }  x {<} l
		\end{cases}$  
		& $ \begin{cases}
			\nd x & \!\!\text{if } l{\leq}\nd a^\trsp \nd x {\leq}u\\
			\nd x - \frac{\nd a(\nd a^{\trsp}\nd x - u)}{\norm{\nd a}^2_2} &\!\! \text{if } \nd a^\trsp \nd x {>} u\\
			\nd x - \frac{\nd a(\nd a^\trsp \nd x - l)}{\norm{\nd a}^2_2} &\!\! \text{if } \nd a^\trsp \nd x {<} l
		\end{cases} $ 
		& $\begin{cases}
			\nd x &\!\! \text{if } l{\leq}\frac{1}{2}\nd x^\trsp\nd x {\leq} u\\
			\frac{\nd x\sqrt{2u}}{\norm{\nd x}} &\!\! \text{if } \frac{1}{2}\nd x^\trsp\nd x {>} u\\
			\frac{\nd x\sqrt{2l}}{\norm{\nd x}} &\!\!\text{if } l{>}\frac{1}{2}\nd x^\trsp\nd x
		\end{cases}$
		& $\begin{cases}
			(\nd x, t), &\!\! \text{if } \norm{\nd x} {\leq} t\\
			(\nd 0, 0), &\!\! \text{if }  \norm{\nd x} {\leq} {-}t\\
			\frac{\norm{\nd x} + t}{2}(\frac{\nd x}{\norm{\nd x}}, 1) &\!\! \text{otherwise}
		\end{cases} $
		\\ \hline
	\end{tabular}
	\label{table:projections}
\end{table*}

In this section, we motivate the use of projections in standard robotic tasks such as hierarchical inverse kinematics and obstacle avoidance.

\subsection{Euclidean projections onto sets}
\label{sub:euc}
The solution $\nd x^*$ to the following constrained optimization problem 
\begin{equation}
	\displaystyle \min_{\nd x}  \norm{\nd x - \nd x_0}_2^2 
	\quad\st\quad \nd x \in \mathcal{C}
	\label{eq:projection}
\end{equation}
is called an Euclidean projection of the point $\nd x_0$ onto the set $\mathcal{C}$ and is denoted as $\nd x^*{=}\Pi_{\mathcal{C}}(\nd x_0)$. This operation determines the point $\nd x \in \mathcal{C}$ that is closest to $\nd x_0$ in Euclidean sense. For many sets $\mathcal{C}$, $\Pi_{\mathcal{C}}(\cdot)$ admits analytical expressions that are given in \Cref{table:projections}. Even though, usually these sets are convex (e.g. bounded domains), some nonconvex sets also admit analytical solution(s) that are easy to compute (e.g. being outside of a sphere). Note that many of these sets are frequently used in robotics, from joint/torque limits and avoiding spherical/square obstacles to satisfying virtual fixtures defined in the task space of the robot.

\addtocounter{footnote}{-1}
\footnotetext{\url{https://hgirgin.github.io/IKSPG.html}}
\addtocounter{footnote}{1}
\footnotetext{\url{https://robotics-codes-from-scratch.github.io/}}
\subsection{Projection view of inverse kinematics}
The inverse kinematics (IK) problem in robotics corresponds to finding a joint configuration $\nd q^*$ of the robot that corresponds to a given desired end-effector pose $\nd p_d$. Iterative procedures are developed to robustly solve this problem considering singularities at the Jacobian level. The success and the convergence speed of these algorithms depend on the initialization of the problem, which is often selected as the current joint configuration of the robot $\nd q_0$. With this view in mind, we can express IK as a projection problem of the initial joint angles $\nd q_0$ onto a set $\mathcal{C}_{\nd q}$ and of the initial end-effector position $\nd p_0$ onto a set $\mathcal{C}_{\nd p}$. These two sets are assumed to be nonempty and closed sets that admit tractable and efficient projections. A common example for $\mathcal{C}_{\nd q}$ is the box constraints for the joint limits. \Cref{fig:ik_examples} shows examples for the set $\mathcal{C}_{\nd p}$ with \Cref{fig:ik_point} showing an equality constraint to a desired point, \Cref{fig:ik_line} shows an affine hyperplane constraint for virtually limiting the robot to be under/on a plane, \Cref{fig:ik_circle} and \Cref{fig:ik_square} show quadric constraints for the end-effector to stay inside/outside or on the boundary of a circle/square. In this work, we exploit these easy projections in a first-order optimization solver with the claim of finding solutions faster than standard constrained optimization problems.

\section{Augmented Lagrangian Spectral Projected Gradient Descent for Robotics}
\label{sec:method}
This section gives the spectral projected gradient descent (SPG) algorithm along with the nonmonotone line search procedure. These algorithms are easy to implement without big memory requirements and yet result in powerful solvers. Next, we give the augmented Lagrangian spectral projected gradient descent (ALSPG) algorithm with extensions to general inequality constraints and multiple projections. 
\subsection{Spectral projected gradient descent}
Spectral projected gradient descent (SPG) is an improved version of a vanilla projected gradient descent using spectral stepsizes. Its excellent numerical results even in comparison to second-order methods have been a point of attraction in the optimization literature \cite{birgin2014spectral}. SPG tackles constrained optimization problems in the form of 
\begin{equation}
	\displaystyle \min_{\nd x} f(\nd x)  
	\quad\st\quad \nd x \in \mathcal{C},
\end{equation}
by constructing a local quadratic model of the objective function 
\begin{align*}
f(\nd x)&\approxeq f(\nd x_k) + {\nabla f(\nd x_k)}^\trsp(\nd x - \nd x_k) + \frac{1}{2\gamma_k}\norm{\nd x - \nd x_k}_2^2, \\
&=\frac{1}{2\gamma_k}\norm{\nd x - (\nd x_k-\gamma_k \nabla f(\nd x_k) )}_2^2 + \text{const.},
\end{align*}
and by minimizing it subject to the constraints as
\begin{equation}
	\displaystyle \min_{\nd x}  \frac{1}{2\gamma_k}\norm{\nd x - (\nd x_k-\gamma_k \nabla f(\nd x_k) )}_2^2
	\quad\st\quad \nd x \in \mathcal{C},
\end{equation}
whose solution is an Euclidean projection as described in \Cref{sub:euc} and given by $\Pi_{\mathcal{C}}(\nd x_k-\gamma_k \nabla f(\nd x_k))$. The local search direction $\nd d_k$ for SPG is then given by
\begin{equation}
\nd d_k = \Pi_{\mathcal{C}}(\nd x_k-\gamma_k \nabla f(\nd x_k)) - \nd x_k,
\end{equation}
which is used in a nonmonotone line search (\Cref{algo:line_search}) with $\nd x_{k+1}=\nd x_k + \alpha_k \nd d_k$, to find $\alpha_k$ satisfying $f(\nd x_{k+1}) \leq f_{\text{max}} + \alpha_k \gamma_k {\nabla f(\nd x_k)}^\trsp \nd d_k$, where $f_{\text{max}}=\max\{f(\nd x_{k-j} ) \: | \: 0 \leq j \leq \min\{k, M - 1\}\}$. Nonmonotone line search allows for increasing objective values for some iterations $M$ preventing getting stuck at bad local minima. 

The choice of $\gamma_k$ affects the convergence properties significantly since it introduces curvature information to the solver. Note that when choosing $\gamma_k=1$, SPG is equivalent to the widely known projected gradient descent. SPG uses spectral stepsizes obtained by a least-square approximation of the Hessian matrix by $\gamma_k\id$. These spectral stepsizes are computed by proposals
\begin{align}
\gamma_k^{(1)}=\frac{\nd s_k^\trsp \nd s_k}{\nd s_k^\trsp \nd y_k} \quad  \text{and} \quad \gamma_k^{(2)}=\frac{\nd s_k^\trsp \nd y_k}{\nd y_k^\trsp \nd y_k}, 
\end{align}
where $\nd s_k=\nd x_{k}-\nd x_{k-1}$ and $\nd y_k=\nabla f(\nd x_{k})-\nabla f(\nd x_{k-1})$ \cite{birgin2014spectral}. In the case of quadratic objective function in the form of $\nd x^\trsp \nd Q \nd x$, these two values correspond to the maximum and minimum eigenvalues of the matrix $\nd Q$. Recent developments in SPG have shown that an alternating use of these spectral stepsizes lead to better performance. The initial spectral stepsize can be computed by setting $\bm{\bar{x}}_{0} = \nd x_{0}-\gamma_{\text{small}}\nabla f(\nd x_0)$, and computing  $\bm{\bar{s}}_{0}=\bm{\bar{x}}_{0}-\nd x_{0}$ and $\bm{\bar{y}}_{0}=\nabla f(\bm{\bar{x}}_{0})-\nabla f(\nd x_0)$. Note that this heuristic operation costs one more gradient computation. The final algorithm is given by \Cref{algo:spg}. 

\begin{algorithm}
	\caption{Non-monotone line search}
	\label{algo:line_search}
	Set $\beta = 10^{-4}$, $\alpha=1$, $M=10$,
	$c = {\nabla f(\nd x_k)}^\trsp \nd d_k$, \\
	$f_{\text{max}}=\max\{f(\nd x_{k-j} ) | 0 \leq j \leq \min\{k, M - 1\}\}$ \\
	\While(){$f(\nd x_k + \alpha\nd d_k) > f_{\text{max}} + \alpha \beta c$}
	{
		
		$\bar{\alpha} = -0.5\alpha^2 c\Big(f(\nd x_k + \alpha\nd d_k) - f(\nd x_k) - \alpha c\Big)^{-1}$ \\
		\eIf {$0.1\leq\bar{\alpha}\leq0.9$}
		{$\alpha=\bar{\alpha}$}{
			{$\alpha=\alpha/2$}}
		
	}
\end{algorithm}
\begin{algorithm}
	\caption{Spectral Projected Gradient Descent (SPG)}
	\label{algo:spg}
	Initialize $\nd x_k$, $\gamma_k$ $\epsilon{=}10^{-5}$, $k{=}0$\;
	\While(){$\norm{\Pi_{\mathcal{C}}(\nd x_k- \nabla f(\nd x_k)) - \nd x_k}_{\infty}>\epsilon$}
	{	
		Find a search direction by
		$\nd d_k = \Pi_{\mathcal{C}}(\nd x_k-\gamma_k \nabla f(\nd x_k)) - \nd x_k$ \\
		Do non-monotone line search using \Cref{algo:line_search} to find $\nd x_{k+1}$ \\
		
		Update the spectral stepsize \\
		$\nd s_{k+1}=\nd x_{k+1}-\nd x_{k}$ and $\nd y_{k+1}=\nabla f(\nd x_{k+1})-\nabla f(\nd x_k)$\\
		$\gamma^{(1)}=\frac{\nd s_{k+1}^\trsp \nd s_{k+1}}{\nd s_{k+1}^\trsp \nd y_{k+1}}$ and $\gamma^{(2)}=\frac{\nd s_{k+1}^\trsp \nd y_{k+1}}{\nd y_{k+1}^\trsp \nd y_{k+1}}$\\
		
		\eIf{$\gamma^{(1)}<2\gamma^{(2)}$}{$\gamma_{k+1} =\gamma^{(2)}$}
		{{$\gamma_{k+1} =\gamma^{(1)}-\frac{1}{2}\gamma^{(2)}$}}
		
		$k = k+1$ 
		
	}
\end{algorithm}
\subsection{Augmented Lagrangian spectral projected gradient descent (ALSPG)}
The SPG algorithm has been shown to be a powerful competition to second-order solvers in many ways. Each iteration can be significantly cheaper than a second-order method if a computationally efficient projection is used and provides better directions than other first-order methods. However, SPG alone is usually not sufficient to solve problems in robotics with complicated nonlinear constraints. In \cite{jia2022augmented}, Jia \emph{et al.} provides an augmented Lagrangian framework to solve problems with constraints $\nd g(\nd x) \in \mathcal{C}$ and $\nd x \in \mathcal{D}$, where $\nd g(\cdot)$ is a convex function, $\mathcal{C}$ is a convex set, and $\mathcal{D}$ is a closed nonempty set, both equipped with easy projections. 

In this section, we build on the work in \cite{jia2022augmented} with the extension of multiple projections and additional general equality and inequality constraints. The general optimization problem that we are tackling here is 
\begin{equation}
	\displaystyle \min_{\nd x \in \mathcal{D}}  f(\nd x)
	\quad\st\quad \nd g_i(\nd x) \in \mathcal{C}_i, \quad \forall i \in \{1,\hdots ,p\}
	\label{eq:general_problem}
\end{equation}
where $\nd g_i(\cdot)$ are assumed to be arbitrary nonlinear functions. Note that even though the convergence results in \cite{jia2022augmented} apply to the case when these are convex functions and convex sets, we found in practice that the algorithm is powerful enough to extend to more general cases. For simplicity, we redefine the additional equality constraints as an additional set to be projected onto with $\mathcal{C}_y=\{\nd y \: | \: \nd y = \nd 0\}$ with $\Pi_{\mathcal{C}_g}(\nd h(\cdot)) =\nd 0$. Also, we transform inequality constraints to equality constraints using the proposed method in the following \Cref{subsec:ineq}.

We use the following augmented Lagrangian function
\begin{align*}
\mathcal{L}(\nd x,\{\nd \lambda^{\mathcal{C}_i}, \rho^{\mathcal{C}_i}\}_{i=1}^p ) &= f(\nd x)+ \\
&\hspace{-2cm}\sum_{i=1}^{p}\frac{\rho^{\mathcal{C}_i}}{2}\bignorm{\nd g(\nd x) + \frac{\nd \lambda^{\mathcal{C}_i}}{\rho^{\mathcal{C}_i}} - \Pi_{\mathcal{C}_i}\Big(\nd g(\nd x) + \frac{\nd \lambda^{\mathcal{C}_i}}{\rho^{\mathcal{C}_i}} \Big)}_2^2
\end{align*}
whose derivative wrt $\nd x$ is given by
\begin{align*}
\nabla \mathcal{L}(\nd x, \{\nd \lambda^{\mathcal{C}_i}, \rho^{\mathcal{C}_i}\}_{i=1}^p) &= \nabla f(\nd x)+ \\
&\hspace{-3cm}\sum_{i=1}^{p}\frac{\rho^{\mathcal{C}_i}}{2}\nabla \nd g_i^\trsp(\nd x)\Big(\nd g_i(\nd x) + \frac{\nd \lambda^{\mathcal{C}_i}}{\rho^{\mathcal{C}_i}} - \Pi_{\mathcal{C}_i}\Big(\nd g_i(\nd x) + \frac{\nd \lambda^{\mathcal{C}_i}}{\rho^{\mathcal{C}_i}} \Big) \Big),
\end{align*}
using the property of convex Euclidean projections derivative $\nabla \norm{\nd g(\nd x)-\Pi(\nd g(\nd x))}_2^2 = \nabla \nd g(\nd x)^\trsp \big(\nd g(\nd x)-\Pi(\nd g(\nd x)\big)$, see \cite{bauschke2011convex} for details. This way, we obtain a formulation which does not need the gradient of the projection function $\Pi_{\mathcal{C}_i}(\cdot)$. One iteration of ALSPG optimizes the subproblem $\argmin_{\nd x\in\mathcal{D}}\mathcal{L}(\nd x, \{\nd \lambda^{\mathcal{C}_i}, \rho^{\mathcal{C}_i}\}_{i=1}^p)$ given $\{\nd \lambda^{\mathcal{C}_i}, \rho^{\mathcal{C}_i}\}_{i=1}^p$, and then updates these according to the next iterate. Defining the auxiliary function $V(\nd x, \nd \lambda^{\mathcal{C}_i}, \rho^{\mathcal{C}_i}){=}\bignorm{\nd g(\nd x)  - \Pi_{\mathcal{C}_i}\Big(\nd g(\nd x) + \frac{\nd \lambda^{\mathcal{C}_i}}{\rho^{\mathcal{C}_i}} \Big)}$, the algorithm is summarized in \Cref{algo:ALSPG}. Note that one can define and tune many heuristics around augmented Lagrangian methods with possible extensions to primal-dual methods. Here, we give only one possible way of implementing ALSPG.

\begin{algorithm}[t]
	\caption{ALSPG}
	\label{algo:ALSPG}
	Set  $\nd \lambda_{0}^{\mathcal{C}_i}{=}\nd 0$, $\rho_{0}^{\mathcal{C}_i}{=}0.1$, $k{=}0$, $\epsilon>0$ \\
	\While(){$\norm{\Delta \mathcal{L}(\nd x, \{\nd \lambda^{\mathcal{C}_i}, \rho^{\mathcal{C}_i}\}_{i=1}^p)}>\epsilon$}
	{
		 $\nd x_{k+1} =\argmin_{\nd x\in\mathcal{D}} \mathcal{L}(\nd x, \{\nd \lambda^{\mathcal{C}_i}, \rho^{\mathcal{C}_i}\}_{i=1}^p)$ with SPG in \Cref{algo:spg} \\
		 
		 \ForEach{$\mathcal{C}_i$}
		 {
		 		$\lambda_{k+1}^{\mathcal{C}_i} = \rho^{\mathcal{C}_i}\Big(\nd g_i(\nd x) + \frac{\nd \lambda^{\mathcal{C}_i}}{\rho^{\mathcal{C}_i}} - \Pi_{\mathcal{C}_i}\Big(\nd g_i(\nd x) + \frac{\nd \lambda^{\mathcal{C}_i}}{\rho^{\mathcal{C}_i}} \Big)\Big) $ \\
		 	
		 	\eIf{$V(\nd x_{k+1}, \nd \lambda_{k+1}^{\mathcal{C}_i}, \rho_k^{\mathcal{C}_i}) \leq V(\nd x_{k}, \nd \lambda_{k}^{\mathcal{C}_i}, \rho_k^{\mathcal{C}_i} ) $}
		 	{
		 		$\rho_{k+1}^{\mathcal{C}_i}{=}\rho_{k}^{\mathcal{C}_i}$ \\
		 	}{
		 		$\rho_{k+1}^{\mathcal{C}_i}{=}10\rho_{k}^{\mathcal{C}_i}$ \\
		 	}

 		}

	}
\end{algorithm}

\subsection{Handling of inequality constraints}
\label{subsec:ineq}
In robotics, one frequent and intuitive way of incorporating the constraints into the optimization problem is to use soft constraints and tune the weights until a satisfactory result is obtained. However, this approach breaks the hierarchy of the task without any real guarantee of constraint satisfaction. Soft constraints are obtained by transforming the hard constraint function into a positive cost function using auxiliary functions such as the barrier function. In this section, we propose to exploit such soft constraint functions as hard constraints to reduce each inequality constraint to an equality constraint, eliminating the need of using slack variables. Note that this procedure is in line with the construction of a standard augmented Lagrangian for inequality constraints. 

Let $g_i(\nd x)  \leq 0$ be the $i^{\text{th}}$ inequality constraint with $i{=}1, \dots, M$ and $g(\cdot): \mathbb{R}^{n} \to \mathbb{R}$. We define $h_i(\nd x) = \max(0, g_i(\nd x)$, where $h(\cdot): \mathbb{R}^{n} \to \mathbb{R}^{+}$. Then, the statement $g_i(\nd x))\leq 0$ is equivalent to $h_i(\nd x)= 0$. Moreover, we can generalize this statement to obtain one single equality constraint from any number of inequality constraints in order to increase the computational speed. This generalization is given by the theorem below.
\begin{theorem}
	\label{th:1}
	The statement $g_i(\nd x)\leq 0, \forall i{=}1, \dots, M$ is equivalent to $h(\nd x) = \sum_{i=1}^M h_i(\nd x) = 0$, where $h_i(\nd x) = \max(0, g_i(\nd x))$.
\end{theorem}
\begin{proof}
	\begin{enumerate}
		\item If  $g_i(\nd x)\leq 0, \forall i{=}1, \dots, M$, then it is by definition that $\sum_{i=1}^M h_i(\nd x) = 0$. 
		\item Assume $\sum_{i=1}^M h_i(\nd x) = 0$ and $\exists j \quad \st \quad g_j(\nd x)>0, \quad \forall j{=}1, \dots, N<M$. Then, $\sum_{i=1}^M h_i(\nd x)=\sum_{j=1}^N h_j(\nd x)=\sum_{j=1}^N g_j(\nd x)>0$, which contradicts the assumption.
	\end{enumerate}
\end{proof}
Although it seems to simplify the problem in terms of dimensions, using \Cref{th:1} to compactly reduce all inequality constraints into one single constraint would result in the loss of some information about the gradients from each constraint in one iteration of any solver. In practice, this presents itself as a trade-off between the number of iterations and the computational complexity of each iteration to solve the optimization problem.

\section{Optimal Control with ALSPG}
\label{spg:oc}
We consider the following generic constrained optimization problem
\begin{equation}
	\displaystyle \min_{\nd x \in \mathcal{C}_{\nd x}, \nd u \in \mathcal{C}_{\nd u}} c(\nd x, \nd u)
	\quad\st\quad
	\begin{array}{l}
		\nd x = \nd F(\nd x_0, \nd u), \\
		\nd h(\nd x, \nd u) = \nd 0,
	\end{array}
	\label{eq:oc_problem}
\end{equation}
where the state trajectory $\nd x{=}\matb{\nd x_1^\trsp, \nd x_2^\trsp, \hdots, \nd x_t^\trsp, \hdots, \nd x_T^\trsp}^\trsp$, the control trajectory $\nd u=\matb{\nd u_0^\trsp, \nd u_1^\trsp, \hdots, \nd u_t^\trsp, \hdots, \nd u_{T-1}^\trsp}^\trsp$ and the function $\nd F(\cdot,\cdot)$ correspond to the forward rollout of the states using a dynamics model $\nd x_{t+1}{=}\nd f(\nd x_t, \nd u_t)$. We use a direct shooting approach and transform \Cref{eq:oc_problem} into a problem in $\nd u$ only by considering
\begin{equation}
		\displaystyle \min_{\nd u \in \mathcal{C}_{\nd u}} c(\nd F(\nd x_0, \nd u), \nd u)   
		\quad\st\quad
		\begin{array}{l}
		\nd F(\nd x_0, \nd u) \in \mathcal{C}_{\nd x}, \\
		\nd h(\nd F(\nd x_0, \nd u), \nd u) = \nd 0,
	\end{array}
	\label{eq:oc_problem2}
\end{equation}
which is exactly in the form of \Cref{eq:general_problem}, if $\nd g_1(\nd u) = \nd F(\nd x_0, \nd u)$ and $\nd g_2(\nd u) = \nd h(\nd F(\nd x_0, \nd u), \nd u)$. The unconstrained version of this problem can be solved with least-square approaches. However, assuming $\nd x_t\in\mathbb{R}^{m}$, $\nd u_t\in\mathbb{R}^{n}$, this requires the inversion of a matrix of size $Tn\times Tn$, whereas here we only work with the gradients of the objective function and the functions $\nd g_i(\cdot)$. The component that requires a special attention is $\nabla \nd F(\nd x_0, \nd u)$ and in particular, its transpose product with a vector. It turns out that this product can be efficiently computed with a recursive formula (as also described in \cite{torrisi2018projected}), resulting in fast SPG iterations. Denoting $\nd A_t=\nabla_{\nd x_t} \nd f(\nd x_t, \nd u_t)$, $\nd B_t=\nabla_{\nd u_t} \nd f(\nd x_t, \nd u_t)$, and $\nabla_{\nd u} \nd F(\nd x_0, \nd u)^\trsp\nd y = \nd z$ with $\nd y=\matb{\nd y_0, \nd y_1, \hdots, \nd y_t, \hdots, \nd y_{T-1}}$, $\nd z=\matb{\nd z_0, \nd z_1, \hdots, \nd z_t, \hdots, \nd z_{T-1}}$, one can show that the matrix vector product $\nabla_{\nd u} \nd F(\nd x_0, \nd u)^\trsp\nd y$ can be written as 

\begin{align*}
	& \matb{
		\nd B_0^\trsp & \nd B_0^\trsp \nd A_1^\trsp & \nd B_0^\trsp \nd A_1^\trsp \nd A_2^\trsp & \hdots & \nd B_0^\trsp \prod_{t=1}^{T-1}\nd A_t^\trsp\\
		\nd 0 & \nd B_1^\trsp & \nd B_1^\trsp \nd A_2^\trsp & \hdots & \nd B_1^\trsp \prod_{t=2}^{T-1}\nd A_t^\trsp \\
		\vdots & \vdots & \vdots & \ddots & \vdots \\
		\nd 0 & \nd 0 & \nd 0 & \hdots & \nd B_{T-1}^\trsp
}\matb{
\nd y_0 \\ \nd y_1 \\ \vdots \\ \nd y_{T-1}
}\\
&= \matb{
	\nd B_0^\trsp(\nd y_0 + \nd A_1^\trsp \nd y_1 + \nd A_1^\trsp \nd A_2^\trsp \nd y_2 +   \prod_{t=1}^{T-1}\nd A_t^\trsp \nd y_{T-1}) \\
	\nd B_1^\trsp(\nd y_1 + \nd A_2^\trsp \nd y_2 + \nd A_2^\trsp \nd A_3^\trsp \nd y_3 +   \prod_{t=2}^{T-1}\nd A_t^\trsp \nd y_{T-1}) \\
	\vdots \\	
	\nd B_{T-2}^\trsp(\nd y_{T-2} +  \nd A_{T-1}^\trsp \nd y_{T-1})\\
	\nd B_{T-1}^\trsp\nd y_{T-1}
},
\end{align*}
where the terms in parantheses can be computed recursively backward by $\bar{\bm{z}}_{t+1}=(\nd y_{t+1} + \nd A_{t}^\trsp\bar{\bm{z}}_{t})$, $\nd z_t=\nd B_{t-1}^\trsp\bar{\bm{z}}_{t}$ and $\bar{\bm{z}}_{T-1}=\nd y_{T-1}$, without having to construct the big matrix $\nabla_{\nd u} \nd F(\nd x_0, \nd u)^\trsp$.

Note that when there are no constraints on the state and $\nd h(\cdot)=0$, \Cref{eq:oc_problem2} can be solved directly with the SPG algorithm. We believe that SPG can be used to solve problems with higher horizons even faster than iLQR. \Cref{fig:ilqr_spg_comparison} shows a breakdown of computational times compared to the number of timesteps for a reaching planning tasks without constraints with a 7-axis manipulator. Here, we plotted the average convergence time in (s) for both algorithms with 5 different end positions in task space and horizons of 100, 1000, 2000, 3000 and 5000 timesteps. iLQR is implemented with dynamic programming. SPG is implemented as detailed in the previous section. Both implementations are in Python. 

\begin{figure}
	\centering
	\includegraphics[width=0.6\columnwidth]{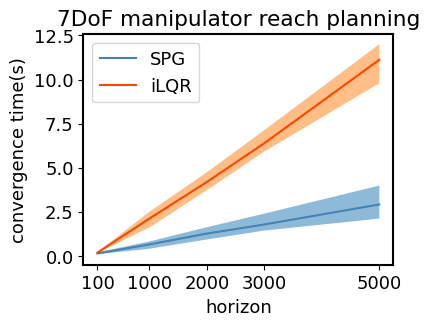}
	\caption{Comparison of iLQR and SPG in terms of convergence time evolution vs the number of timesteps or horizon.}
	\label{fig:ilqr_spg_comparison}
\end{figure}
\section{Convex polytope projections and linear transformations}
Often, Euclidean projection problems are composed of a linear transformation of the constraint set onto which the projection is easy, hindering the analytical projection property of this set. ALSPG can be used directly to solve this kind of problems in a very efficient way, still exploiting the projection capability of the base constraint set. For example, consider a unit second-order cone set as $\mathcal{C}_{\text{SOC}}=\{(\nd z, t) \: | \: \norm{\nd z}_2\leq t\}$ and a generic second-order cone (SOC) constraint as $\mathcal{C}{=}\{\nd x \: | \: \norm{\nd A \nd x + \nd b}_2 \leq \nd c^\trsp \nd x + d \}$. This can be transformed to a unit second-order cone set by taking $\nd g(\nd x){=}\matb{\nd A \nd x + \nd b & \nd c^\trsp \nd x + d}$ and therefore $\mathcal{C}{=}\{\nd x \: | \: \nd g(\nd x) \in \mathcal{C}_{\text{SOC}} \}$. Then the optimization problem of projection onto a generic second-order cone, namely, $\argmin_{\nd x} \norm{\nd x - \nd x_0}_2^2 \quad \st \quad  \norm{\nd A \nd x + \nd b}_2 \leq \nd c^\trsp \nd x + d$ can be rewritten as $\argmin_{\nd x} \norm{\nd x - \nd x_0}_2^2 \quad \st \quad \nd g(\nd x) \in \mathcal{C}_{\text{SOC}}$ and can be solved efficiently using the ALSPG algorithm with only unit second-order cone projections, without requiring explicit derivatives of the cone constraints.

In the case of convex polytope projections, we can even find some conditions when the linear transformation does not break the analytical projections. Especially for rectangular projections, which are a special case of convex polytope projections, we can find special conditions such that we can still find analytical expressions even if we rotate and scale the rectangles. In the next section, we give the development and insights of convex polytope projections as these are one of the most commonly encountered constraint types in robotics problems such obstacle avoidance. We then explain what kind of linear transformations can be applied to projections to preserve their analytical projection capability.

\subsection{Convex polytope projections}
A convex polytope of $n$ sides can be described by $n$ lines with slopes $\nd a_i$ and intercepts $b_i$. The inside region of this polytope (e.g. for reaching) is given by ``\emph{and}'' constraints $\mathcal{C}_{\text{polytope}}^{\text{in}}=\{ \nd x | \bigwedge_{i=0}^n\nd a_i^\trsp \nd x \leq u_i\}$ while the outside region (e.g. for obstacle avoidance) is given by its negative statement with ``\emph{or}'' constraints $\mathcal{C}_{\text{polytope}}^{\text{out}}=\{\nd x | \bigvee_{i=0}^n\nd a_i^\trsp \nd x > l_i$\}. The projection onto $\mathcal{C}_{\text{polytope}}^{\text{in}}$ can be described as a summation of $n$ hyperplane projections in ALSPG. Even though constraints for the set $\mathcal{C}_{\text{polytope}}^{\text{out}}$ can not be easily described in general optimization solvers, we can show that the projection of a point $\nd x_0$ onto this set requires finding the closest hyperplane $i$ to $\nd x_0$, then the projection outside the hyperplane with index $i$, namely $\Pi^i_{\mathcal{C}_{\text{polytope}}^{\text{out}}}(\nd x_0)$. The minimum value of the objective function of the projection becomes $\norm{\Pi^i_{\mathcal{C}_{\text{polytope}}^{\text{out}}}(\nd x_0)- \nd x_0}$ which is equal to the distance of $\nd x_0$ to the hyperplane $i$ (one can check this by inserting the corresponding values from \Cref{table:projections}). This observation makes significant simplifications for the solvers that can take projections into account. 

In this section, we give the simplifications of this idea for often-encountered rectangular regions. The constraint of being inside a square region, also called a box constraint, can be described by infinity norms as the set $\mathcal{C}_{\text{rect}}^{\text{in}}{=}\{\nd x \: | \:  \norm{\nd x}_{\infty} \leq u \}$ represents the inside region of a square of width $u$ centered at the origin. This is basically a compact description of 4 lines (in 2D) describing the square, i.e., $x\leq u$, $-u\leq x$, $y\leq u$, $-u\leq y$. This observation allows us to write down $\mathcal{C}_{\text{rect}}^{\text{out}}{=}\{x \: | \: l \leq \norm{\nd x}_{\infty}\}$ which represents the outside region of a square of width $l$ centered at the origin. $\mathcal{C}_{\text{rect}}^{\text{in}}$ is a simple clipping operation for $\nd x_0$ as described in \Cref{table:projections}. However, $\mathcal{C}_{\text{rect}}^{\text{out}}$ requires setting up the optimization problem for the Euclidean projection and checking the KKT conditions. For conciseness, we give here only the resulting projection. Denoting $k$ the index where $k{=}\argmax_i \abs{\nd x_{0,i}}$, the projection onto $\mathcal{C}_{\text{rect}}^{\text{out}}$ is then given by 
\begin{equation}
	\Pi_{\mathcal{C}}(\nd x_0)_j = 		
	\begin{cases}
		\nd x_{0,j} & \!\!\text{if } \nd x_{0,k} {\leq} l,\\
		l\sign{\nd x_{0,k}}  & \!\!\text{otherwise} . 
	\end{cases}
\end{equation}

\subsection{Linear transformation of projections}
Having stated projections for some basic geometric primitives, one may need to apply rotation and translation operations to such shapes to exploit more complex ones. One such example is the transformation of square projections onto rotated and translated square regions. Considering a convex set $\mathcal{C}=\{\nd x | f(\nd x) \leq t\}$, one can show that the projection onto $\mathcal{C}^{\prime}=\{\nd x | f(\nd A(\nd x - \nd x_c)) \leq t\}$ is given by $\Pi_{\mathcal{C}^{\prime}}(\nd x_0)=\nd A^{\inv}\Pi_{\mathcal{C}}(\nd A(\nd x_0 - \nd x_c))+\nd x_c$, where $\nd A$ is an orthogonal matrix. For creating rectangular regions, one needs to scale each dimension of the variable, i.e., multiplying by a diagonal matrix. Even though this does not generalize to all cases, for the rectangular regions, one can show that $\nd A$ can be in the form of a multiplication of an orthogonal matrix and a diagonal matrix. For example, while a square of length $L$ can be described by the set $\mathcal{C}=\{\nd x | \norm{\nd x}_{\infty}=L/2 \}$, a rectangle of length $L$ and width $W$, which is rotated by an angle $\theta$ can be described with the transformation matrix $\nd A=\nd R(\theta)\nd D$, where $\nd R$ is the rotation matrix, and $\nd D=\diag(1, L/W)$. 

\section{Experiments}
\label{sec:experiments}
In this section, we perform experiments solving inverse kinematics problems, motion planning and MPC for a task with hybrid dynamics, and motion planning for rectangular obstacle avoidance. The motivation behind these experiments is to show that: 1) the proposed way of solving these robotics problem can be unconventionally faster than the second-order methods such as iLQR; and 2) exploiting projections whenever we can, instead of leaving the constraints for the solver to treat them as generic constraints, increases the performance significantly.

\subsection{Constrained inverse kinematics}
A constrained inverse kinematics problem can be described in many ways using projections. One typical way is to find a $\nd q \in \mathcal{C}_{\nd q}$ that minimizes a cost to be away from a given initial configuration $\nd q_0$ while respecting general constraints $\nd h(\nd q) = \nd 0$ and projection constraints $\nd f(\nd q) \in \mathcal{C}_{\nd x}$
\begin{equation}
	\min_{\nd q \in \mathcal{C}_{\nd q}} \norm{\nd q - \nd q_0}_2^2  
	\quad\st\quad 
	\begin{array}{l}
		\nd h(\nd q) = \nd 0, \\
		\nd f(\nd q) \in \mathcal{C}_{\nd x},
	\end{array}
\end{equation}
where $\nd f(\cdot)$ can represent entities such as the end-effector pose or the center of mass for which the constraints are easier to be expressed as projections onto $\mathcal{C}_{\nd x}$, and $\mathcal{C}_{\nd q}$ can represent the configuration space within the joint limits. \Cref{fig:ik_examples} shows a 3-axis planar manipulator with $\nd f(\cdot)$ representing the end-effector position and $\mathcal{C}_{\nd x}$ denoting (a) $\mathcal{C}_{\nd x}=\{\nd x \: | \: \nd x{=}\nd x_d\}$, (b) $\mathcal{C}_{\nd x}=\{\nd x \: | \: \nd a^\trsp \nd x + b{=}0\}$, (c) $\mathcal{C}_{\nd x}=\{\nd x \: | \: \leq r_i^2 \leq \norm{\nd x-\nd x_d}_2^2 \leq r_o^2\}$ and (d) $\mathcal{C}_{\nd x}=\{\nd x \: \: | \: \: \norm{\nd x-\nd x_d}_{\infty, \nd W} \leq L \}$. We applied the ALSPG algorithm iteratively to obtain a reactive control loop and we implemented it on an interactive webpage running Python, see \Cref{fig:ik_examples}.

\textbf{Talos IK:} We tested our algorithm on a high dimensional (32 DoF) inverse kinematics problem of TALOS robot (see \Cref{fig:talos}) subject to constraints: i) center of mass inside a box; ii) end-effector constrained to lie inside a sphere; and iii) foot position and orientations are given. We compared two versions of ALSPG algorithm: 1) by casting these constraints as projections onto $\mathcal{C}_{\nd x}$; and 2) by keeping all the constraints inside the function $\nd h(\cdot)$ to see direct advantages of exploiting projections in ALSPG. We ran the algorithm from 1000 different random initial configurations for both cases and compared the number of function and Jacobian evaluations, $n_f$ and $n_j$. For case 1), we obtained $n_f{=}897.64\pm82.84$ and $n_j{=}883.44\pm 81.7$, while for case 2), we obtained
$n_f{=}6459.4\pm 3756.8$ and $n_j{=}3791.79\pm 1061.05$.

\textbf{Robust IK:} In this experiment, we would like to achieve a task of reaching and staying in the half-space under a plane whose slope is stochastic because, for example, of the uncertainties in the measurements of the vision system. The constraint can be written as $\nd a^\trsp \nd f(\nd q) \leq 0$, where $\nd a \sim \mathcal{N}(\nd \mu, \nd \Sigma)$. We can transform it into a chance-constraint to provide some safety guarantees in a probabilistic manner. The idea is to find a joint configuration $\nd q$ such that it will stay under a stochastic hyperplane with a probability of $\eta\geq0.5$. This inequality can be written as a second-order cone constraint wrt $\nd f(\nd q)$ as $\nd \mu^\trsp \nd f(\nd q) + \Psi^{-1}(\eta) \norm{\nd \Sigma^{\frac{1}{2}}\nd f(\nd q)}_2 \leq 0$, where $\Psi(\cdot)$ is the cumulative distribution function of zero mean unit variance Gaussian variable. Defining $\nd g(\nd q)=\matb{(\nd \Sigma^{\frac{1}{2}}f(\nd q))^\trsp & \nd \mu^\trsp \nd f(\nd q)}^\trsp$, the optimization problem can then be defined as 
\begin{equation}
	\displaystyle \min_{\nd q \in \mathcal{C}_{\nd q}} \norm{\nd q - \nd q_0}_2^2
	\quad\st\quad 
	\nd g(\nd q) \in \mathcal{C}_{\text{SOC}},
\end{equation}
which can be solved efficiently without using second-order cone (SOC) gradients, by using the proposed algorithm. We tested the algorithm on the 3-axis robot shown in \Cref{fig:soc_manip} by optimizing for a joint configuration with a probability of $\eta=0.8$ and then computing continuously the constraint violation for the last 1000 time steps by sampling a line slope from the given distribution. We obtained a constraint violation percentage of around 80\%, as expected.
\begin{figure}
	\centering
	\subcaptionbox{\label{fig:talos}}
	[0.4\columnwidth]{\includegraphics[width=0.4\columnwidth]{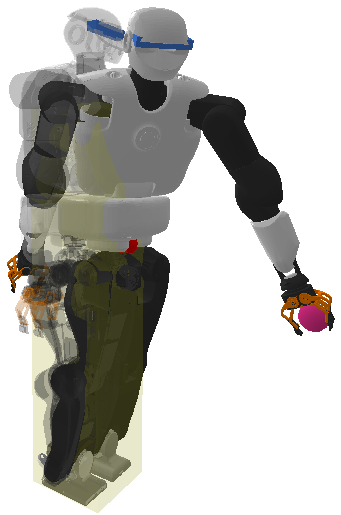}}
	\subcaptionbox{\label{fig:soc_manip}}[0.4\columnwidth]{\includegraphics[width=0.4\columnwidth]{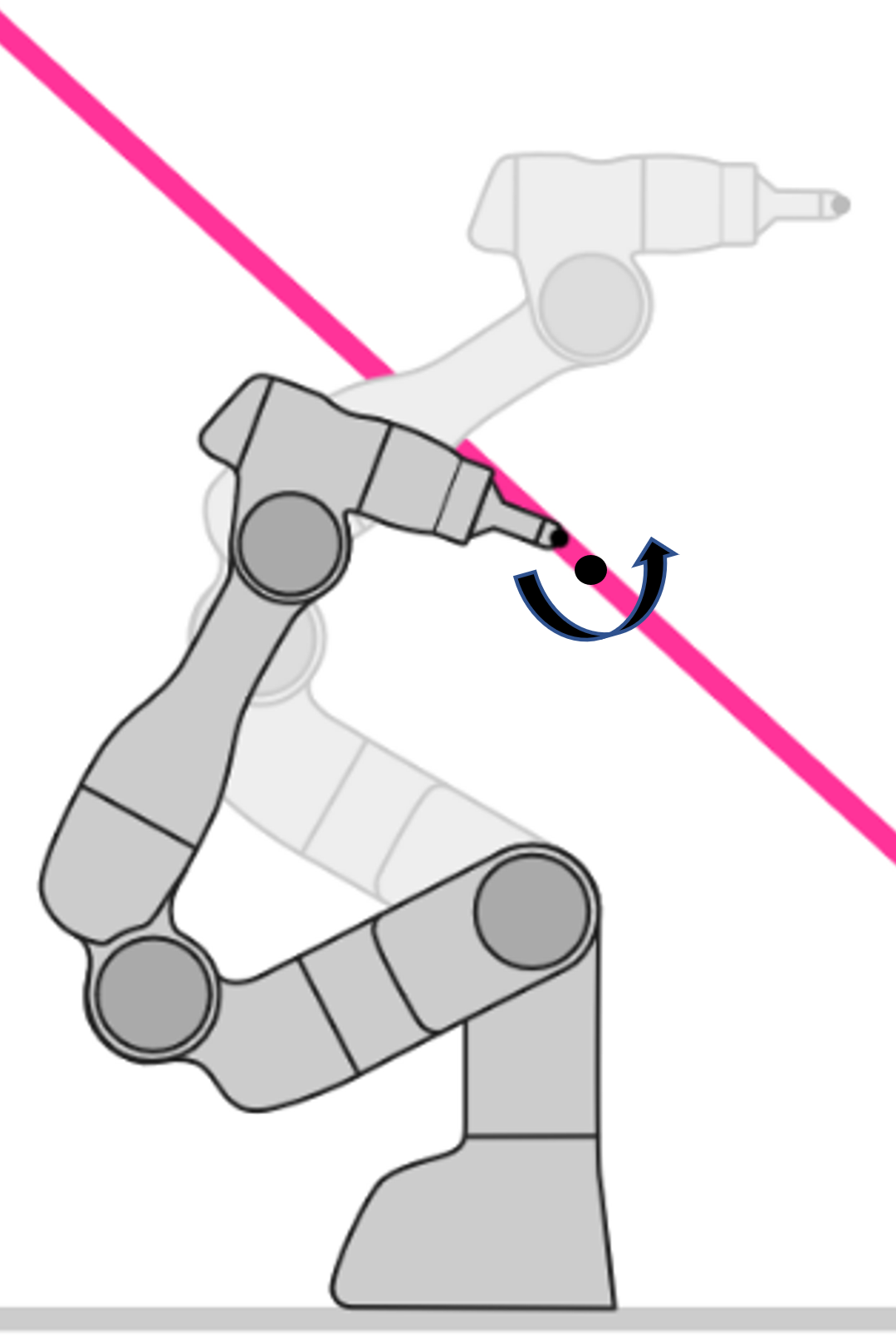}}
	\vspace{2mm}
	\caption{Inverse kinematics problems solved with the proposed algorithm. (a) Talos inverse kinematics problem with foot pose, center of mass stability (red point inside yellow rectangular prism) and end-effector inside a (pink) sphere constraints. (b) Robust inverse kinematics solution with $\mathcal{C}_{\nd p}=\{\nd p \: | \:\nd \mu^\trsp \nd p + \Psi^{-1}(\eta) \norm{\nd \Sigma^{\frac{1}{2}}\nd p}_2 \leq 0$\}.}
\end{figure}

\subsection{Motion planning and MPC on planar push}
Non-prehensile manipulation has been widely studied as
a challenging task for model-based planning and control,
with the pusher-slider system as one of the most prominent
examples (see \Cref{fig:push_spg}). The reasons include hybrid dynamics with various interaction modes, underactuation and contact uncertainty. In this experiment, we study motion planning and MPC on this planar push system, without any constraints, to compare to a standard iLQR implementation. Motion planning convergence results for 10 different tasks are given for iLQR and ALSPG along with means and variances in \Cref{fig:push_plan_vs}. Although iLQR seems to converge to medium accuracy faster than ALSPG, because of the difficulties in the task dynamics, it seems to get stuck at local minima very easily. On the other hand, ALSPG seems to perform better in terms of variance and local minima. We applied MPC with iLQR and ALSPG with an horizon of 60 timesteps and stopped the MPC as soon as it reached the goal position with a desired precision. \Cref{table:push_mpc_vs} shows this comparison in terms of convergence time (s), number of function evaluations and number of Jacobian evaluations. According to these findings, ALSPG performs better than a standard iLQR, even when there are no constraints in the problem.
\begin{table}[]
	\centering
	\caption{Comparison of MPC with iLQR and ALSPG for planar push}
	\begin{tabular}{l|l|l|}
		\cline{2-3}
		& iLQR              & ALSPG              \\ \hline
		\multicolumn{1}{|l|}{Convergence time (s)}   & $14.5 \pm 1.3$         & $2.9 \pm 0.5$      \\ \hline
		\multicolumn{1}{|l|}{Number of function ev.} & $26689.5 \pm 1830.5$ & $6104.0 \pm 1455.6$ \\ \hline
		\multicolumn{1}{|l|}{Number of jacobian ev.} & $225.9 \pm 27.4$       & $78.4 \pm 8.5$      \\ \hline
	\end{tabular}
	\label{table:push_mpc_vs}
\end{table}

\begin{figure}
	\centering
	\begin{subfigure}{0.48\columnwidth}
		\includegraphics[width=\columnwidth]{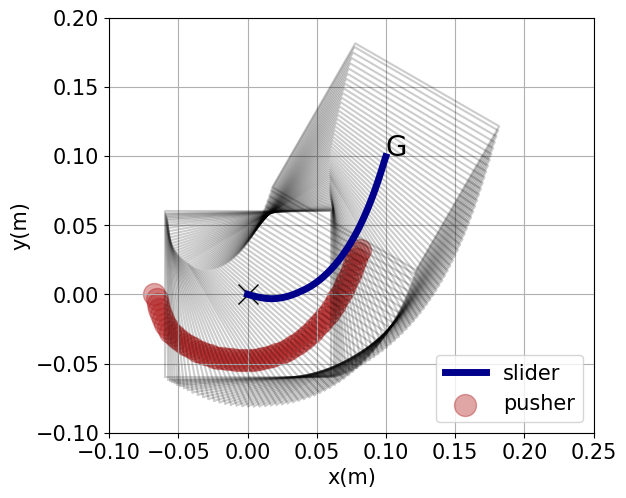}
		\caption{Pusher-slider system path optimized by the proposed algorithm to go from the state $(0,0,0)$ to $(0.1, 0.1, \pi/3)$. Optimal control solved with SPG results in a smooth path for the pusher-slider system.}
	\end{subfigure}\hfill
	\begin{subfigure}{0.48\columnwidth}
		\centering
		\includegraphics[width=0.87\columnwidth]{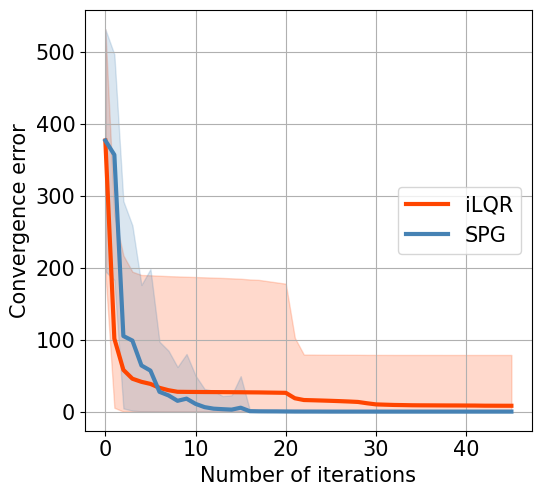}
		\caption{Convergence error mean and variance plot for iLQR and ALSPG motion planning algorithm for 10 different goal conditions starting from the same initial positions and control commands.}
		\label{fig:push_plan_vs}
	\end{subfigure}
	\vspace{2mm}
	\caption{SPG algorithm applied to a pusher-slider system.}
	\label{fig:push_spg}
\end{figure}

\subsection{Motion planning with obstacle avoidance}
Obstacle avoidance problems are usually described using  geometric constraints. In robot manipulation tasks, capsules and spheres are typically used to represent the robot and the environment, such that the shortest distance computations and their gradients can be computed efficiently. In autonomous parking tasks, obstacles and cars are usually described as 2D rectangular objects. In this experiment, we take a 2D double integrator point car reaching a target pose in the presence of rectangular obstacles (see \Cref{fig:obstacle}). We apply ALSPG algorithm with and without projections (ALSPG-Proj. and  ALSPG-WoProj.) to illustrate the main advantages of having an explicit projection function over direct constraints. The main difference is without projections, the solvers need to compute the gradient of the constraints, whereas with projections, this is not necessary. In order to understand the differences between first-order and second-order methods, we also compared ALSPG-Proj to AL-SLSQP with projections (SLSQP-Proj.), which is the same algorithm except the subproblem is solved by a second-order solver SLSQP from Scipy. 
The objective function is $c(\nd x, \nd u)=10^{-1}\norm{\nd x_T -\nd x_T^{\text{G}}}_{2}^2 + 10^{-4}\norm{\nd u_T}_{2}^2$. We performed 5 experiments, each with different settings of 4 rectangular obstacles and compared the convergence properties. The results are given in \Cref{table:obstacle_vs}. The comparison between ALSPG-Proj. and  ALSPG-WoProj. reports a clear advantage of using projections instead of plain constraints in the convergence properties. Although the convergence time comparison is not necessarily fair for SPG implementations as the SLSQP solver calls C++ functions, the comparison of ALSPG-Proj. and  SLSQP-Proj. shows that ALSPG-Proj. still achieves lower convergence time. 

\subsection{MPC on 7-axis manipulator}
We tested ALSPG algorithm on the MPC problem of tracking an object with box constraints on the end-effector position of a Franka Emika robot (see \Cref{fig:exp_setup}). An Aruco marker on the object is tracked by a camera held by another robot. In this experiment, the goal is to show the real-time applicability of the proposed algorithm for a constrained problem in the presence of disturbances. In \Cref{fig:mpc_box}, the error of the constraints and the objective function is given for a 1 min. time period of MPC with a short-horizon of 50 timesteps. Between 20s and 30s, the robot is disturbed by the user thanks to the compliant torque controller run on the robot. We can see that the algorithm drives smoothly the error to zero, see accompanying video. 

\begin{figure}
	\centering
	\includegraphics[width=0.9\columnwidth]{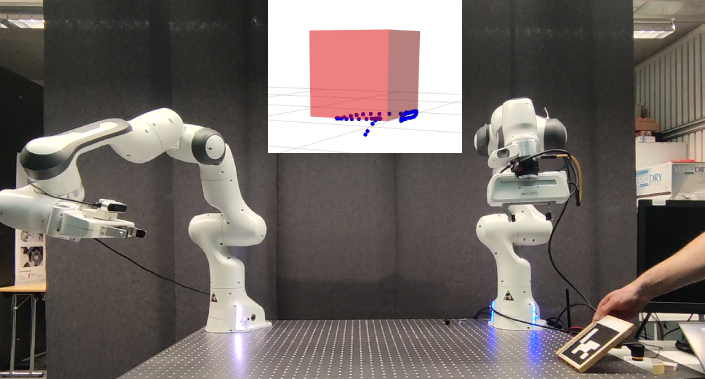}
	\caption{MPC setup for tracking an object subject to box constraints.}
	\label{fig:exp_setup}
\end{figure}

\begin{figure}
	\centering
	\includegraphics[width=\columnwidth]{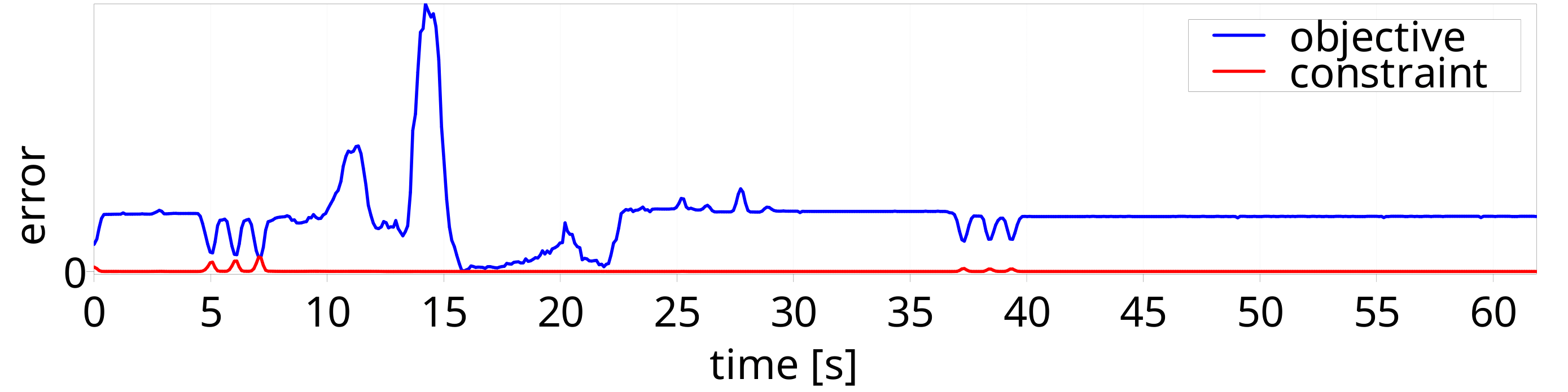}
	\caption{Error in the objective and the squared norm of the constraint value during 1 min execution of MPC on Franka Emika robot.}
	\label{fig:mpc_box}
\end{figure}

\begin{figure}
	\centering
	\includegraphics[width=0.5\columnwidth]{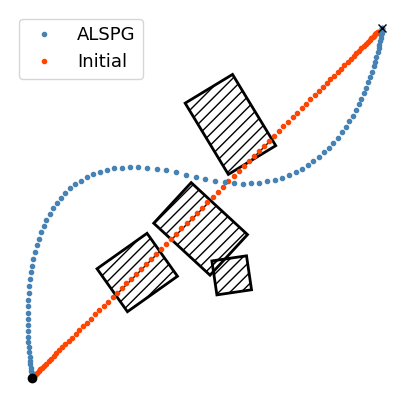}
	\caption{Motion planning problem in the presence of 4 scaled and rotated rectangular obstacles.}
	\label{fig:obstacle}
\end{figure}

\begin{table}[]
	\centering
	\caption{Comparison of motion planning for obstacle avoidance for three cases}
	\setlength\tabcolsep{0.3pt}
	\begin{tabular}{l|l|l|l|}
		\cline{2-4}
		& ALSPG-Proj.  & SLSQP-Proj.  & ALSPG-WoProj.           \\ \hline
		\multicolumn{1}{|l|}{Convergence time (ms)}   & $392.0 \pm 66.2$         & $679.0 \pm 260.0$ & $2780.0 \pm 530.9$      \\ \hline
		\multicolumn{1}{|l|}{Number of function ev.} & $332.0 \pm 60.6$ & $438.4 \pm 200.7$ & $571.8 \pm 109.9$ \\ \hline
		\multicolumn{1}{|l|}{Number of jacobian ev.} & $187.8 \pm 34.6$       & $389.2 \pm 164.3$ & $399.0 \pm 93.9$      \\ \hline
	\end{tabular}

	\label{table:obstacle_vs}
\end{table}

\section{Conclusion}
\label{sec:spg_conclusion}
In this work, we presented a fast first-order constrained optimization framework based on geometric projections, and applied it to various robotics problems ranging from inverse kinematics to motion planning. We showed that many of the geometric constraints can be rewritten as a logic combination of geometric primitives onto which the projections admit analytical expressions. We built an augmented Lagrangian method with spectral projected gradient descent as subproblem solver for constrained optimal control. We demonstrated: 1) the advantages of using projections when compared to setting up the geometric constraints as plain constraints with gradient information to the solvers; and 2) the advantages of using spectral projected gradient descent based motion planning compared to a standard second-order iLQR algorithm through different robot experiments. 

Sample-based MPC have been increasingly popular in recent years thanks to its fast practical implementations, despite their lack of theoretical guarantees. In contrast, second-order methods for MPC require a lot of computational power but with somewhat better convergence guarantees. We argue that ALSPG, being already in-between these two methodologies in terms of these properties, promises great future work to combine it with sample-based MPC to further increase its advantages on both sides.

\bibliographystyle{IEEEtran}
\bibliography{bib_spg}

\end{document}